\documentclass[fleqn,10pt]{wlscirep}
\usepackage{graphicx}
\usepackage[utf8]{inputenc}
\usepackage[T1]{fontenc}
\usepackage{amsmath}
\usepackage{amsthm}
\newtheorem{theorem}{Theorem}
\usepackage{float}
\usepackage{subcaption}
\usepackage{multirow}
\usepackage{xurl}   %
\title{DMDSC: A Dynamic-Margin Deep Simplex Classifier for Open-Set Recognition on Medical Image Datasets}

\author[1]{Vishal}
\author[1]{Arnav Aditya}
\author[1,*]{Nitin Kumar}
\author[1]{Saurabh J. Shigwan}

\affil[1]{Department of Computer Science and Engineering,
Shiv Nadar Institution of Eminence Deemed to be University,
Delhi NCR, India}

\affil[*]{Correspondence to: nitin.kumar@snu.edu.in}

\begin{abstract}
Medical imaging datasets are often characterized by extreme class imbalances, where rare pathologies are significantly underrepresented compared to common conditions. This imbalance poses a dual challenge for Open-Set Recognition (OSR): models must maintain high classification accuracy on known classes while reliably rejecting unknown samples unseen during training in the clinical settings.
While recently proposed Deep Simplex Classifier (DSC)~\cite{cevikalp2024reaching} and UnCertainty-aware Deep Simplex Classifier (UCDSC)~\cite{Aditya_2026_WACV} successfully leverage Neural Collapse to ensure maximal inter-class separation, they rely on a uniform margin that does not account for the varying densities of medical classes. 
In this paper, we propose DMDSC, an enhanced framework featuring a dynamic margin approach. Our approach automatically adapts class-specific margins based on label frequency, imposing proportionally larger margins to rare pathologies within the inter-class and auxiliary background separation losses. This promotes stronger separation from rival class centers and auxiliary background samples for minority classes. Apart from this, we asymptotically show the upper bounds on the margins as a function of the number of classes.  Extensive experiments conducted on diverse medical benchmarks on BloodMNIST\cite{medmnistv2}, OCTMNIST\cite{medmnistv2}, DermaMNIST\cite{medmnistv2}, BreaKHis$(40\times)$~\cite{spanhol2015dataset} and ASC\cite{naqvi2023augmented} datasets, demonstrate that our framework performs better/equivalent to the state-of-the-art methods.
\end{abstract}
\begin{document}

\flushbottom
\maketitle
\thispagestyle{empty}

\section{Introduction}
\label{sec:intro}
Open-set Recognition (OSR) addresses the fundamental requirement for classifiers to operate safely in real-world clinical environments by detecting samples from unknown classes that were not seen during training \cite{geng2020recent, scheirer2012toward}. In medical imaging, this capability is essential, as a model must correctly classify known pathologies while rejecting imaging artifacts or novel diseases to prevent overconfident misdiagnoses and ensure trustworthy AI-assisted decision-making \cite{bendale2015towards}. Recent studies have utilized the phenomenon of Neural Collapse (NC)\cite{papyan2020prevalence}, in which feature representations geometrically converge to the vertices of a simplex equiangular tight frame (ETF), forming a structured feature space that increases inter-class separation \cite{ cevikalp2024reaching}.

\begin{figure}[t]
    \centering
    \includegraphics[width=0.8\linewidth]{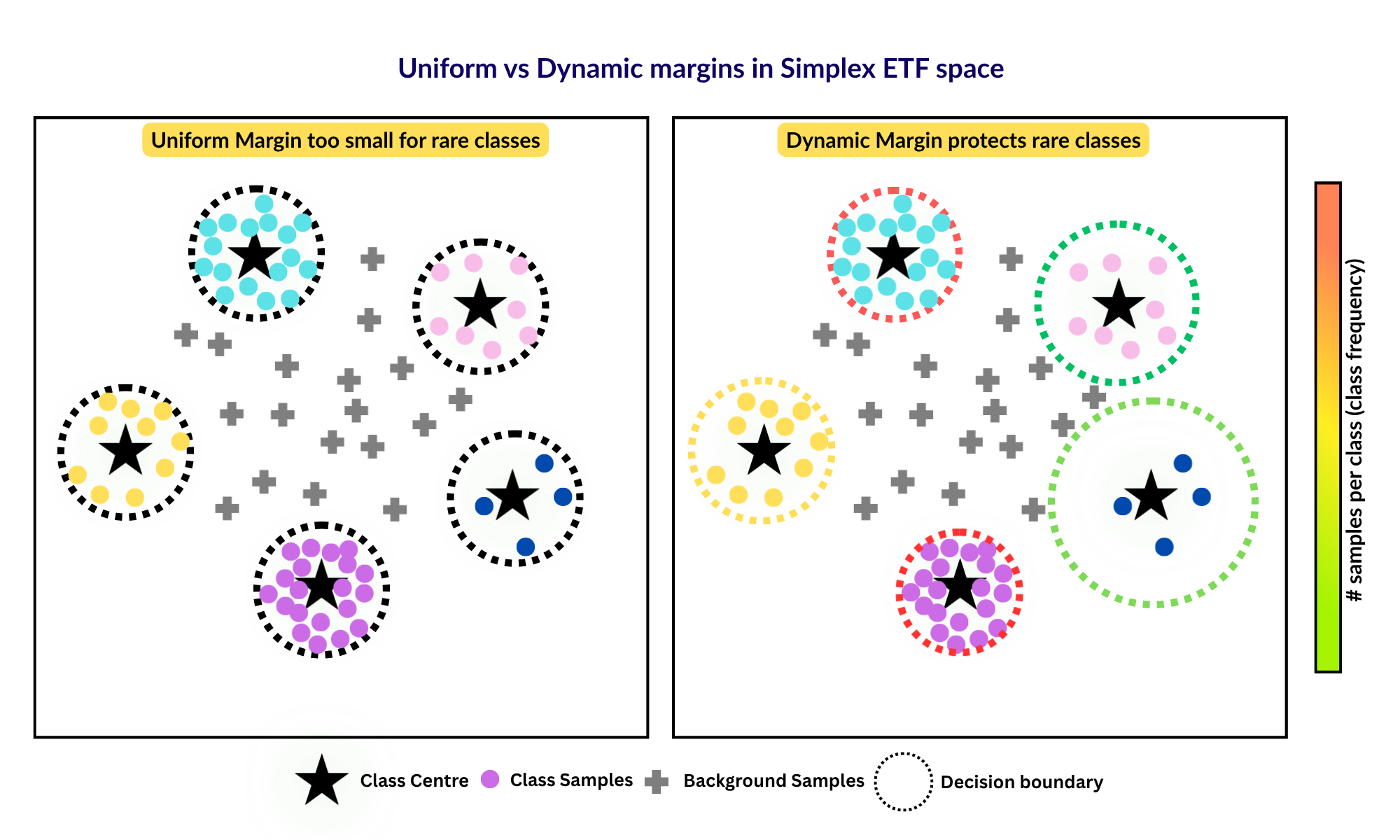}
        \caption{Comparison of uniform and dynamic margins in the simplex ETF space. (Left) The uniform margin assigns the same margin to all classes, resulting in insufficient separation for rare classes. (Right) The dynamic margin adapts the class-specific margin inversely to class frequency, assigning larger margins to rare classes and smaller margins to frequent classes, thereby promoting class-aware feature separation. Decision boundary colors correspond to class frequency 
        (shown by the colorbar): rare classes (green tones) receive larger protective margins, while frequent classes (red/yellow tones) maintain tighter boundaries.}
    \label{fig:simplex}
\end{figure}

Despite the theoretical elegance of NC-based frameworks, significant challenges persist in the medical domain, most notably the presence of class imbalance \cite{japkowicz2002class}. Many medical datasets \cite{medmnistv2} typically exhibit the phenomenon of class imbalance with some diseases having ample data and others having minimal sample coverage.  
In imbalanced scenarios, uniform-margin constraints incorporated by DSC \cite{cevikalp2024reaching} and UCDSC \cite{Aditya_2026_WACV} are suboptimal, where majority classes can dominate the embedding space \cite{japkowicz2002class}, encroaching on the representations of rare pathologies and increasing the open space risk \cite{scheirer2012toward} around minority class centers.

Further, classification methods such as LDAM \cite{cao2019learning}, which are based on class-aware-margin, do not take into account the important criterion of the structure of the embedded space. 
Unlike the LDAM dynamic margin, our proposed margin framework differs not only in formulation but also provides range of values, which is bounded by the radius $R$ of the hypersphere circumscribing the simplex ETF. 
This bounds-based approach on the margins help to prevent adverse effects on the decision boundaries of frequent classes.
Also in LDAM~\cite{cao2019learning}, margins are derived using binary class assumption, while in our approach the vertices of underlying geometric structure (simplex ETF) naturally acts as class prototypes in multi-class framework.

Our proposed \textbf{DMDSC} method
automatically adapts class-specific margins based on label frequency during training. By enforcing larger margins for underrepresented pathologies, DMDSC penalizes the open space more aggressively around rare class centers, ensuring that even minority classes maintain discriminative feature clusters that are resilient to encroachment by majority classes.

We verify this approach through a comprehensive empirical analysis across a set of 5 medical datasets. We observe that DMDSC outperforms other state-of-the-art methods. 
Furthermore, we provide a comparative study against existing SOTA methods across three primary metrics: AUROC for unknown class detection, OSCR for joint recognition and classification, and Accuracy (ACC) for closed-set performance. Our results show that DMDSC consistently achieves superior or at-par performance across these metrics, particularly in the presence of class imbalance.
An overview of the proposed method is shown in Figure~\ref{fig:framework}

\begin{figure*}[t]
    \centering
    \includegraphics[width=1.0\linewidth]{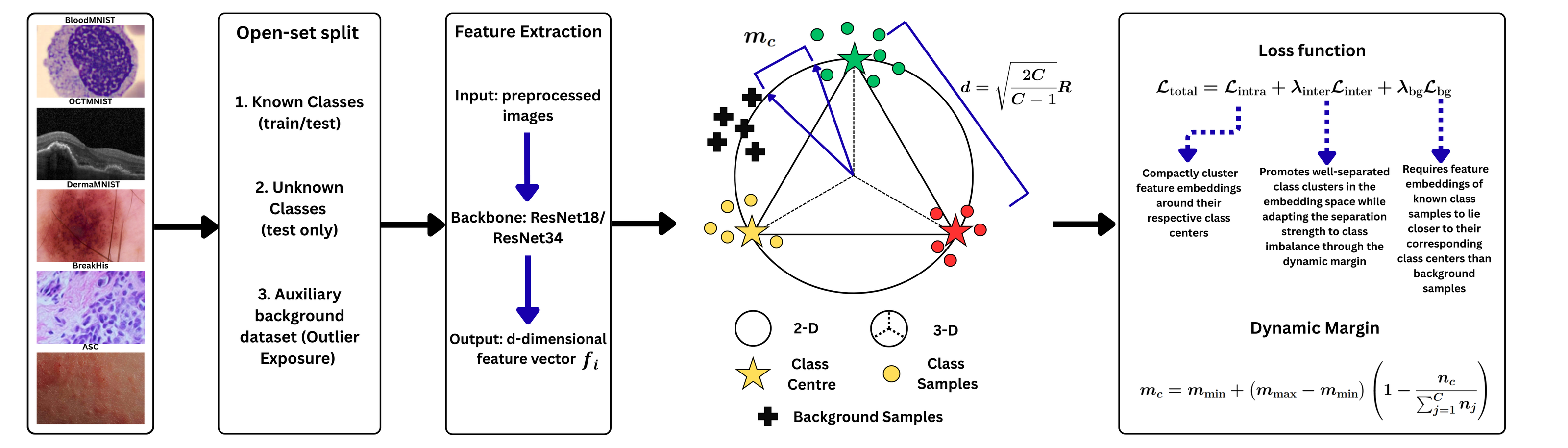}
        \caption{Illustration of the proposed open-set recognition framework with dynamic margin learning. In our proposed framework, input images are mapped to a hyperspherical embedding space where class centers (stars) are fixed at the vertices of a simplex ETF. The distance between any two class centers is d, ensuring maximal inter-class separation. Known-class samples (dots) are pulled toward their corresponding centers, while background/outlier samples (crosses) are pushed away. The intra-class loss promotes compact clustering around the corresponding class centers, while the class-adaptive dynamic margin strengthens inter-class and background/outlier separation. The total loss jointly promotes discriminative features and robust unknown-class rejection.}
    \label{fig:framework}
\end{figure*}

The main contributions of this paper are summarized as follows:
\begin{itemize}[label=\textbullet]
    \item We formulate a novel \textbf{inter-class loss} that incorporates dynamic-margin and enforces class-dependent separation between each feature representation and simplex ETF class centers, reducing open-space risk for minority classes.
    \item We propose \textbf{DMDSC}, an imbalance-robust OSR framework that maintains stable geometric separation for both majority and minority pathologies within a simplex ETF framework.
    \item We demonstrate through extensive experimentation on clinical benchmarks that our dynamic approach provides superior or at-par rejection capabilities compared to other state-of-the-art (SOTA) methods, including the severe class imbalance conditions.
\end{itemize}

\subsection*{Theoretical Foundations}
Scheirer et al. \cite{scheirer2012toward} first formalized OSR by introducing the concept of \textit{open space risk}, defining a classifier $f(x) > 0$ to indicate a known-class decision and minimizing an objective that includes both empirical risk and the risk over open space:
\begin{equation}
    R_{\text{open}}(f) = \lambda R_\varepsilon(f) + R_O(f),
\end{equation}
where $R_\varepsilon$ is the empirical classification loss and $R_O$ measures how much of the open space is incorrectly labeled as known. This formulation encourages classifiers to shrink decision regions around training data to avoid incorrect labeling of unknowns.
This foundational work led to the development of Compact Abating Probability (CAP) models \cite{scheirer2014probability} and the Extreme Value Machine (EVM) \cite{rudd2018extreme}. These methods utilize Extreme Value Theory (EVT) to model the "tails" of distance distributions, allowing for a statistical basis for novelty detection. While EVT-based methods provide strong mathematical guarantees, they often struggle with the high-dimensional manifolds generated by modern deep neural networks.

Our approach aligns with statistical learning principles by constraining decision regions around the observed training data to reduce open space risk. By incorporating class-aware margins into the simplex equiangular tight frame (ETF) framework, we bridge the gap between geometric feature alignment and the practical challenges imposed by class imbalance in clinical data.

This theoretical framework penalizes classifiers for labeling regions far from the training data as known classes, effectively defining the "unknown" as the vast, unoccupied regions of the feature space.  

\section{Related Work}
\label{sec:relatedwork}

\subsection{Neural Collapse and Geometric Constraints}
Another foundational concept is \textbf{Neural Collapse (NC)} \cite{papyan2020prevalence} describes a geometric convergence during the terminal phase of training where last-layer features collapse to their class means, and these means align with the vertices of a Simplex Equiangular Tight Frame (ETF). This property provides the structural backbone for the Deep Simplex Classifier (DSC) approach \cite{cevikalp2024reaching}. Subsequently, UCDSC \cite{Aditya_2026_WACV} was built upon this by integrating uncertainty-aware regularization to penalize the open space between class centres. While UCDSC utilized NC to align the known classes, it inherited a notable drawback: the assumption of a \textit{uniform margin} across all classes. In real-world data distributions, particularly those with high semantic overlap or class imbalance, a uniform margin often fails to protect the boundaries of minority classes, an architectural constraint we challenge in this work.

\subsection{Reconstruction, Generative, and Contrastive Models}
Reconstruction-based models like CROSR \cite{yoshihashi2019classification} and C2AE \cite{oza2019c2ae} utilize decoders to measure novelty via reconstruction error. Similarly, generative models such as G-OpenMax \cite{ge2017generative} and OSRCI \cite{neal2018open} synthesize "pseudo-unknowns" to train the classifier. ConOSR \cite{xu2023contrastive} uses supervised contrastive learning with soft targets and Mixup-generated virtual examples to improve open-set representation learning. Unlike reconstruction and generative methods, which often introduce significant computational overhead or risk bias from synthetic data quality, ConOSR focuses on learning more discriminative feature representations. Compared with these approaches, our method is purely discriminative and enforces structure directly through the geometric arrangement of features and dynamic margins.

\subsection{Prototype and Reciprocal Point Learning}
Prototype-based methods like GCPL \cite{yang2018robust} and Reciprocal Point Learning (RPL) \cite{chen2020learning} attempt to structure the embedding space by clustering features around centers or "reciprocal points."
A significant weakness of RPL \cite{chen2020learning} and its adversarial extension ARPL \cite{chen2021adversarial} is their reliance on \textit{uniform-margin constraints}. These methods do not account for label frequency, making them sensitive to initialization in imbalanced medical datasets. In contrast, our proposed DMDSC introduces a class-aware dynamic margin that scales inversely with sample frequency, specifically protecting the boundaries of rare pathologies.

\subsection{OSR in Medical Imaging}
In clinical settings, OSR is uniquely challenging due to the "Near-OOD" (Out-of-Distribution) problem, where unknown pathologies share significant semantic or structural overlap with known conditions, making them indistinguishable via simple distance metrics \cite{yang2024generalized}. %

Methods like Open-Margin Cosine Loss (OMCL) \cite{liu2023learning} utilize Margin Loss with Adaptive Scale (MLAS) and Open Space Suppression (OSS) to penalize sparse embedding regions. More recently, SCI-OSR \cite{10634167} introduced a semi-supervised framework using selective mixup strategies to handle imbalanced medical datasets. However, these approaches either rely on the heuristic that unknowns reside solely in sparse areas or based on pseudo unseen class samples. Such methods can struggle to maintain sharp geometric separation in "Near-OOD" tasks where unknown pathologies are proximal to known clusters. In contrast, our method avoids synthetic data and generative overhead. By integrating a class-aware dynamic margin directly into the rigid Neural Collapse objective, we ensure robust detection of rare conditions.

\section{Proposed Method}

Given the labeled training dataset $\mathcal{D} = \{(x_i, y_i)\}_{i=1}^{n}$  consisting of $n$ samples from $C$ known classes $\phi_C=\{1,\dots,C\}$, where $x_i \in \mathcal{X}$ represents the $i$-th input sample and $y_i \in \phi_C$ is the corresponding class label. In addition, let $\mathcal{D}_{bg} = \{x_k^{bg}\}_{k=1}^{K}$ denote an auxiliary dataset containing samples that do not belong to any of the known classes. We consider a deep neural network parameterized by $\theta$ that implements a feature embedding 
$f_{\theta}:\mathcal{X}\to\mathbb{R}^d$, which maps each input $x$ to a $d$-dimensional latent vector 
$f_{\theta}(x)$. The embedding dimension satisfies $d \ge C-1$ \cite{cevikalp2024reaching}. Our proposed method learns the network parameters $\theta$ using $\mathcal{D}$ and $\mathcal{D}_{bg}$ such that for a test sample $x$ the classification rule is given by:
\begin{equation}
\hat{y}(x) \;=\;
\begin{cases}
\hat{c}(x), & \text{if } \min_{c}\|f_{\theta}(x)-s_c\|_2 \le \tau,\\
C+1, & \text{otherwise}.
\end{cases}
\label{eq:open_set_rule}
\end{equation}
where
\begin{equation}
\hat{c}(x) \;=\; \arg\min_{c\in\{1,\dots,C\}}\bigl\| f_{\theta}(x)-s_c\bigr\|_{2},
\label{eq:closed_set_pred}
\end{equation}

This rule correctly classifies known samples and rejects unknowns. The class prototypes $\{s_c\}_{c=1}^{C}\subset\mathbb{R}^{d}$ are fixed at the vertices of a simplex equiangular tight frame (ETF) inscribed in a hypersphere of radius $R$~\cite{cevikalp2024reaching,papyan2020prevalence}, with embedding dimension $d\ge C-1$. We follow standard practice and report threshold-independent metrics (AUROC, OSCR) in addition to closed-set accuracy, so that $\tau$ does not bias the comparison. 
The threshold values of $\tau$, expressed relative to 
$R$, are chosen empirically.

Let the deep neural network feature representations of the training samples be denoted by
$\{(f_i, y_i)\}_{i=1}^{n}$, where $f_i \in \mathbb{R}^d$ is the learned feature
vector and $y_i \in \phi_C$ is the corresponding class label. Under this setting, the overall training objective of the proposed classifier can be written as:
\begin{equation}
\mathcal{L}_{total}
= \mathcal{L}_{intra} 
+ \lambda_{inter} \mathcal{L}_{inter}
+ \lambda_{bg} \mathcal{L}_{bg}
\end{equation}
The weight parameters $\lambda_{\text{inter}}$ and $\lambda_{\text{bg}}$ must be set by the user. We obtained the best results for $\lambda_{\text{inter}}$ and $\lambda_{\text{bg}}$ between 0 and 10.

We define intra-class loss $\mathcal{L}_{intra}$, dynamic margin-based losses $\mathcal{L}_{inter}$ and $\mathcal{L}_{bg}$ \cite{chin2018robust} as follows:
\begin{equation}
\mathcal{L}_{intra} = \frac{1}{n} \sum_{i=1}^{n} \left\| f_i - s_{y_i} \right\|_2^2,
\end{equation}
where $s_{y_i}$ denotes the predefined class prototype corresponding to
label $y_i$, constructed as a simplex vertex. This loss term encourages the learned feature embeddings to be compactly clustered around their respective class centers during training.

To enforce separation among the known classes, we introduce an inter-class
triplet-style loss that pushes each feature embedding away from all rival class centers. For a feature vector $f_i$ belonging to class $y_i$, this loss
encourages the distance to the correct class center $s_{y_i}$ to be smaller than the
distance to any rival class center $s_c$, $c \neq y_i$, by at least a class-adaptive
margin $m_{y_i}$. Formally, the inter-class loss is defined as:
{\small
\begin{equation}
\mathcal{L}_{inter}
=
\sum_{i=1}^{n}
\sum_{\substack{c=1 \\ c \neq y_i}}^{C}
\max\!\Big(
0,\;
m_{y_i}^2
+ \| f_i - s_{y_i} \|_2^2
- \| f_i - s_c \|_2^2
\Big),
\label{eq:inter_triplet}
\end{equation}
}
which promotes well-separated class clusters in the embedding space while adapting
the separation strength to class imbalance through the dynamic margin $m_{y_i}$.

 Recent studies \cite{cevikalp2023anomaly, dhamija2018reducing, geng2020recent, miller2021class} show that incorporating an auxiliary (background) dataset containing samples from class outside the known target classes can substantially improve open-set recognition performance. Motivated by these findings, we use "300k Random Images" \cite{torralba200880, hendrycks2019oe} as auxiliary data, a small subset of the publicly available 80 Million Tiny Images dataset \cite{torralba200880}, and integrate it into our proposed loss function to enforce separation between unknown regions and known-class centers. Let the deep neural network features $f^{\mathrm{bg}}_k \in \mathbb{R}^d$, $k=1,\dots,K$, be the features of the auxiliary samples ($K$ is the number of auxiliary samples). To incorporate these auxiliary samples during training, we introduce an additional loss term $\mathcal{L}_{bg}$ to push background features away from the known-class centers:
{\small
\begin{equation}
\mathcal{L}_{bg}
=
\sum_{i=1}^{n}
\sum_{k=1}^{K}
\max\!\Big(
0,\;
m_{y_i}^2
+ \|f_i - s_{y_i}\|_2^2
- \|f^{\mathrm{bg}}_k - s_{y_i}\|_2^2
\Big),
\label{eq:bg_loss}
\end{equation}
}
where $m_{y_i}$ denotes the class-adaptive dynamic margin parameter associated with class $y_i$. This loss term includes a margin constraint, requiring that feature embeddings of known class samples to lie closer to their corresponding class centers than background samples by a minimum margin $m_{y_i}$.

\subsection{Class Adaptive Dynamic Margin}
Existing open-set recognition methods such as DSC \cite{cevikalp2024reaching} and UCDSC \cite{Aditya_2026_WACV} use a uniform margin for all classes to push background samples away from the known class clusters. Although this uniform margin works well for balanced datasets, it leads to a performance degradation on imbalanced datasets, where majority classes form dense clusters and minority classes form sparse clusters. Using a uniform margin for all classes leads to two key issues; $i)$ the margin may be too small for minority classes, which needs more separation to avoid being confused with unknown samples. $ii)$ the margin may be unnecessarily too large for majority classes, which causes unstable training. To address these limitations, we propose a class-adaptive dynamic margin $m_c$ for any class $c$ as:  
\begin{equation}
m_{c}
=
m_{\min}
+
\bigl(m_{\max} - m_{\min}\bigr)
\left(
1 - \frac{n_c}{\sum_{j=1}^{C} n_j}
\right),
\label{eq:dynamic_margin}
\end{equation}
where $n_c$ denotes the number of training samples belonging to class $c \in C$, and
$C$ is the total number of known classes. The minimum margin $m_{\min}$ and maximum margin $m_{\max}$ are hyperparameters constrained as:
\begin{equation}
0 < m_{\min} < m_{\max} < \frac{R}{\sqrt{2}}
\label{eq:margin_constraint}
\end{equation}
where $R$ denotes the radius of the hypersphere on which the class centers (simplex ETF vertices) are embedded. The theoretical properties of the class-adaptive dynamic margin and its geometric upper bound are derived in Theorems~\ref{thm:class-adaptive-margin-rule} and \ref{thm:margin_bound}, respectively. %

It is worth noting that although our proposed dynamic margin concept looks similar to DMCL \cite{lin2025dynamic}, unlike DMCL \cite{lin2025dynamic} which uses pairwise margin $m_{ij}$ for any pair of classes $i$ and $j$, our proposed dynamic margin $m_c$ is a class specific margin which adjusts according to the number of samples in each class in DSC framework. This formulation assigns larger margins to minority classes and smaller margins to majority classes, thereby accounting for class imbalance in the feature space.
Also, unlike UCDSC which combines a fixed-margin background loss with an uncertainty-aware regularizer, DMDSC introduces an explicit inter-class loss term with a class specific dynamic margin which is incorporated in both the inter-class ($\mathcal{L}_{inter}$) as well as auxiliary-background ($\mathcal{L}_{bg}$) terms in the loss function.

\begin{theorem}[Class-adaptive margin]
\label{thm:class-adaptive-margin-rule}
Let $n_c>0$ be the number of training samples in class $c$, let $N=\sum_{j=1}^{C} n_j$, $p_c=n_c/N$, and define
\begin{equation}\label{eq:margin-rule} m_c = m_{\min} + (m_{\max}-m_{\min})(1-p_c), \end{equation}
where $0<m_{\min}<m_{\max}$. Then $m_c\in[m_{\min},m_{\max})$ for all $c$. Moreover, if $n_a\ge n_b$, then $m_a\le m_b$, and $\lim_{p_c\to 1} m_c = m_{\min}$, $\lim_{p_c\to 0} m_c = m_{\max}$. Finally, \eqref{eq:margin-rule} is the unique affine function satisfying $m(1)=m_{\min}$ and $m(0)=m_{\max}$.
\end{theorem}

\begin{proof}
Since $0<p_c\le 1$, we have $0\le 1-p_c<1$. Hence $m_{\min}\le m_c < m_{\min}+(m_{\max}-m_{\min})=m_{\max}$, which proves the range claim. If $n_a\ge n_b$, then $p_a\ge p_b$, so $m_a\le m_b$ because $m_c$ is linear and decreasing in $p_c$. The limit statements follow directly from \eqref{eq:margin-rule}. For uniqueness, let $m(p)=\alpha+\beta p$ be affine. The conditions $m(0)=m_{\max}$ and $m(1)=m_{\min}$ give $\alpha=m_{\max}$ and $\beta=m_{\min}-m_{\max}$, hence $m(p)=m_{\max}+(m_{\min}-m_{\max})p = m_{\min}+(m_{\max}-m_{\min})(1-p)$.
\end{proof}
\begin{theorem}[Disjointness of classes]
\label{thm:margin_bound}
Let $\{s_i\}_{i=1}^C \subset \mathbb{R}^d$ be simplex-ETF class centers on a hypersphere of radius $R$. Under feature collapse, define $B(s_i,m_{\max})=\{x\in\mathbb{R}^d:\|x-s_i\|\le m_{\max}\}$. If $m_{\max}<R\sqrt{\frac{C}{2(C-1)}}$, then the balls $B(s_i,m_{\max})$ are pairwise disjoint. In particular, as $C\to\infty$, a sufficient asymptotic condition is $m_{\max}<\frac{R}{\sqrt{2}}$.
\end{theorem}

\begin{proof}
For $i\neq j$, if $x\in B(s_i,m_{\max})\cap B(s_j,m_{\max})$, then by the triangle inequality, $\|s_i-s_j\|\le \|s_i-x\|+\|x-s_j\|\le 2m_{\max}$. Thus the balls are disjoint whenever $2m_{\max}<\|s_i-s_j\|$. For simplex-ETF centers \cite{papyan2020prevalence}, $\|s_i-s_j\|=R\sqrt{\frac{2C}{C-1}}$, so it suffices that $m_{\max}<\frac{R}{2}\sqrt{\frac{2C}{C-1}}=R\sqrt{\frac{C}{2(C-1)}}$. Taking $C\to\infty$ yields the asymptotic bound $m_{\max}<R/\sqrt{2}$.
\end{proof}

\section{Experiments, Results and Discussions}
\label{sec:exp_res_disc}

\subsection{Dataset}
We evaluate our proposed method on MedMNIST v2 benchmark \cite{medmnistv2}, the BreaKHis histopathology dataset \cite{spanhol2015dataset}, and the Augmented Skin Conditions (ASC) dataset \cite{naqvi2023augmented}. We use the official $28{\times}28$ MedMNIST v2 images, and the official training/test splits are utilized, whereas BreaKHis$(40\times)$ and ASC images are resized to $224{\times}224$ to preserve fine-grained visual characteristics. Sample images from the five medical image datasets used in our experiments are shown in Figure \ref{fig:Sample_images}. To facilitate open-set recognition (OSR) evaluation, each dataset is split into "known" and "unknown" classes for each trial. Known classes are used for training and closed-set validation, while unknown classes are held out until inference to assess the model's rejection capabilities.

\subsubsection{{BloodMNIST} \cite{medmnistv2}} It consists of $17{,}092$ color microscopic images of individual normal blood cells categorized into $8$ classes  \cite{acevedo2020dataset}. The images are obtained from individuals without infection or hematologic disorders.

\subsubsection{{OCTMNIST} \cite{medmnistv2}} It is derived from a retinal optical coherence tomography (OCT) dataset  \cite{kermany2018deep} and contains $109{,}309$ grayscale images representing $4$ diagnostic classes. To simulate real-world clinical scenarios, the "healthy" class is consistently included in the known set.

\subsubsection{{DermaMNIST} \cite{medmnistv2}} It is sourced from the HAM10000 dataset \cite{codella2019skin, tschandl2018ham10000}. It is a diverse, multi-source collection of $10{,}015$ dermatoscopic images representing $7$ common skin diseases.

\subsubsection{{BreaKHis$(40\times)$} \cite{spanhol2015dataset}}
(Breast Cancer Histopathological Image Classification) contains 7{,}909 breast tumor images from 82 patients, with both benign and malignant samples across multiple magnifications. The images are categorized into 8 histological subclasses: four benign tumor types, namely adenosis (A), fibroadenoma (F), phyllodes tumor (PT) and tubular adenoma (TA), and four malignant tumor types, namely ductal carcinoma (DC), lobular carcinoma (LC), mucinous carcinoma (MC) and papillary carcinoma (PC). In this work, we use only the $40\times$ subset, which contains 1{,}995 images: 625 benign (A: 114, F: 253, PT: 109, TA: 149) and 1{,}370 malignant (DC: 864, LC: 156, MC: 205, PC: 145). All 8 subclasses are treated as separate classes for open-set recognition, and the known and unknown classes in each trial are drawn from them. The $40\times$ subset contains 81 unique patient identifiers, as one patient is associated with two subclasses. We use a patient-wise 70:30 train-test split, with 57 patients for training and 24 for testing, ensuring no patient overlap between the two sets.

\subsubsection{Augmented Skin Conditions (ASC)~\cite{naqvi2023augmented}}
In addition to the above evaluation datasets, we use the ASC  dataset for a controlled imbalance study. The ASC dataset contains $2{,}394$ images ($399$ per class) and focuses on enhanced images of $6$ common conditions. The dataset is divided into training and testing sets using an $80{:}20$ split. 

\subsubsection{Auxiliary Background Data}
We use the "300k Random Images" dataset~\cite{hendrycks2019oe} as an auxiliary background data to improve unknown class discrimination, which is a cleaned subset of the 80 Million Tiny Images dataset~\cite{torralba200880}, with all images appearing in the CIFAR, Places, and LSUN classes removed. During training, each batch of known-class samples is paired with a background batch of the same batch size, which is randomly selected without replacement from auxiliary dataset. The same auxiliary dataset and sampling protocol are used for DSC, UCDSC, and DMDSC across all five medical datasets and across all runs to ensure a fair comparison.

\begin{figure*}
    \centering
    \includegraphics[width=0.8\linewidth]{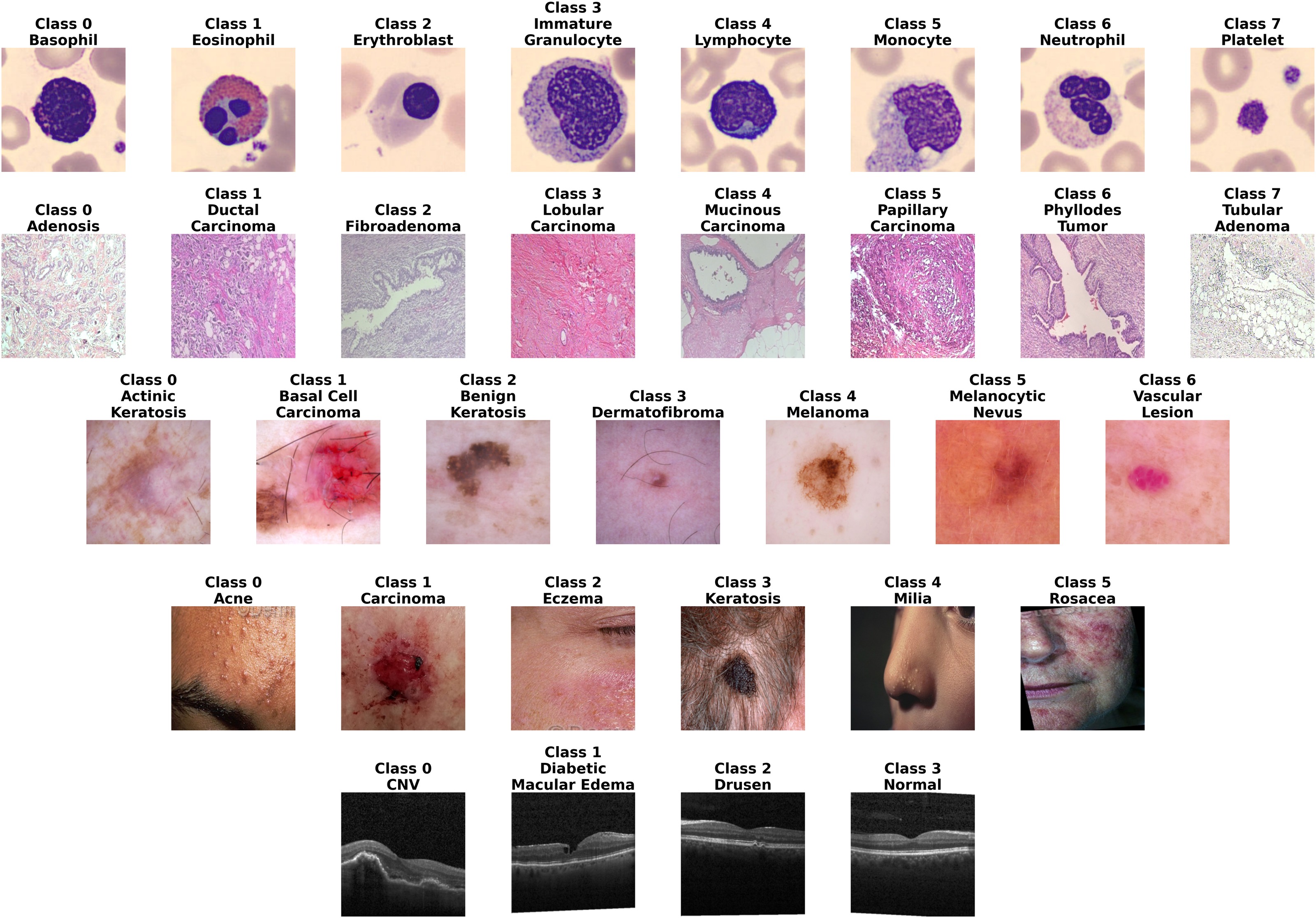}
    \caption{Sample images from the five datasets (a) BloodMNIST (b) BreaKHis$(40\times)$ (c) DermaMNIST (d) Augmented Skin Conditions (ASC) (e) OCTMNIST}
    \label{fig:Sample_images}
\end{figure*}

\begin{figure*}[t]
    \centering
    \includegraphics[width=1.0\linewidth]{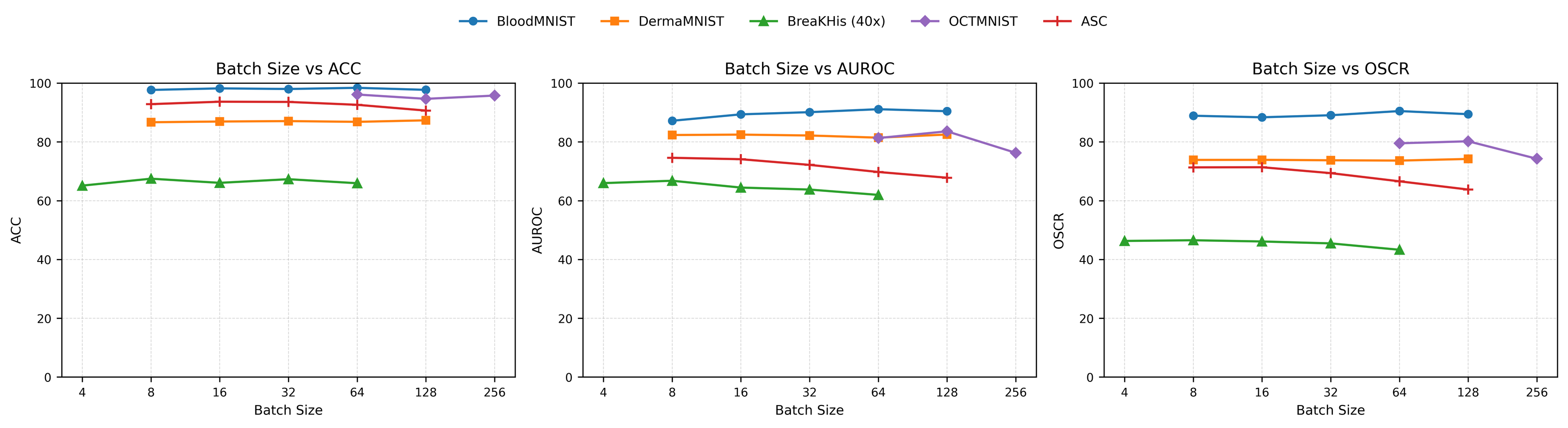}
    \caption{Effect of batch size on ACC, AUROC, and OSCR across five datasets.}
    \label{fig:batch_size}
\end{figure*}

\begin{figure*}[t]
    \centering
    \includegraphics[width=1.0\linewidth]{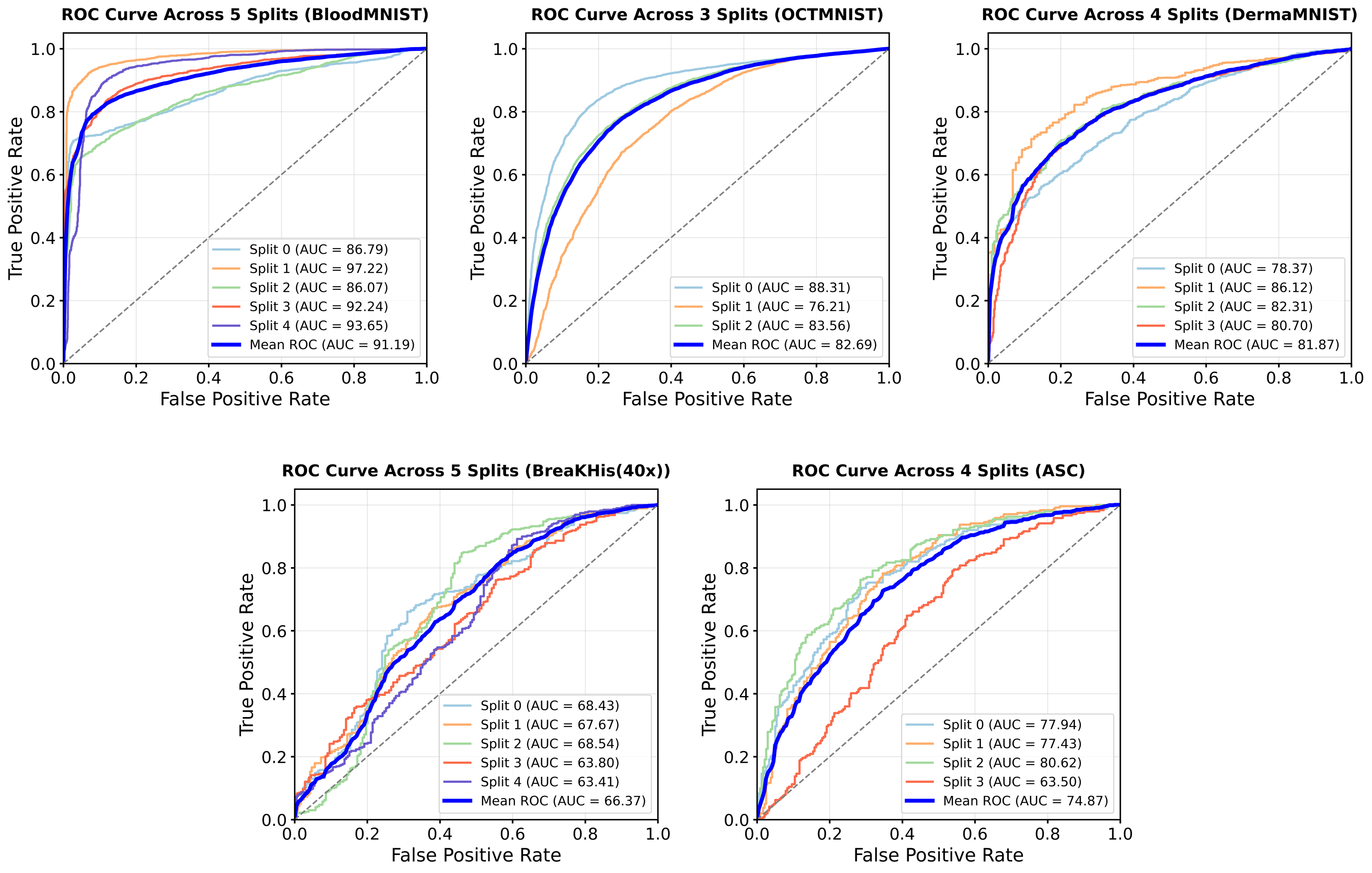}
    \caption{Receiver Operating Characteristic (ROC) curves for all five datasets from a single run.}
    \label{fig:roc_curves}
\end{figure*}

\subsection{Evaluation Metrics} 
 We use accuracy (ACC) to evaluate closed-set classification performance. To measure open-set performance, we report AUROC, which is a threshold-independent metric that reflects how well the model distinguishes known and unknown samples. We also use OSCR, which jointly considers both closed-set accuracy and open-set recognition ability. A higher OSCR value indicates better overall performance in open-set settings.

\subsection{Implementation Details}
We use ResNet18 \cite{he2016deep} as the backbone classification network for BloodMNIST, DermaMNIST, BreaKHis$(40\times)$, ASC, and ResNet34 \cite{he2016deep} for OCTMNIST. The models are trained using RMSprop with an initial learning rate $10^{-4}$. 
The batch size is set to 64 for BloodMNIST and DermaMNIST, 128 for OCTMNIST, 8 for BreaKHis$(40\times)$, and 16 for ASC. The effect of varying batch sizes on ACC, AUROC, and OSCR across all five datasets is illustrated in Figure~\ref{fig:batch_size}.
 The models are trained for 200 epochs on BloodMNIST and DermaMNIST, and 100 epochs on OCTMNIST, BreaKHis$(40\times)$, and ASC. The hyperparameter values are determined through cross-validation. All the experiments are conducted using $k$-fold $(k= 10)$ validation. Scores reported are averaged across all folds.
During training, images are augmented using random cropping, random horizontal flipping, and normalization to improve generalization. We run our experiments on NVidia L4 Tensor core GPU with 24GB of memory. For each dataset, a single experimental run takes approximately $2$ to $2.5$ hours.

Except for the ASC dataset, all other datasets are naturally imbalanced, so we also demonstrate the robustness of the proposed \textbf{DMDSC} method on the \textbf{ASC} dataset by introducing artificial class imbalance. The imbalance is measured using the \textbf{Imbalance Ratio (IR)}, which is defined as the ratio of the number of samples in the majority class to the number of samples in the minority class\cite{japkowicz2002class, wang2022geometric}. Experiments are conducted at $\text{IR}=10$ and $\text{IR}=100$, and performance is compared using ACC, AUROC, and OSCR.

\subsection{Results}
\subsubsection{Comparison with SOTA Methods}
Table \ref{tab:comparison with SOTA} presents a comprehensive comparison between DMDSC and existing state-of-the-art (SOTA) open-set recognition methods across BloodMNIST, OCTMNIST, DermaMNIST, BreaKHis$(40\times)$, and ASC using ACC, AUROC, and OSCR averaged over five random runs, where DMDSC achieves the best or competitive open-set recognition performance across the evaluated datasets. On BloodMNIST, DMDSC achieves the highest AUROC (91.12\%) and OSCR (90.49\%), outperforming strong baselines such as DSC, UCDSC, OMCL, ARPL+CS and ConOSR. On OCTMNIST, it reports the best AUROC (83.60\%) and competitive OSCR (80.21\%) with ConOSR OSCR (80.82\%). For DermaMNIST, DMDSC achieves the highest AUROC (81.44\%) and OSCR (73.64\%), outperforming ConOSR, which achieves an AUROC (79.50\%) and an OSCR of (73.11\%). Furthermore, on the fine-grained BreaKHis$(40\times)$ dataset, DMDSC achieves the highest AUROC (66.78\%) and OSCR (46.54\%), highlighting its robustness under challenging histopathological variations. On ASC, DMDSC achieves the highest ACC (93.67\%) and OSCR (71.38\%), while attaining a competitive AUROC (74.11\%), slightly below UCDSC (74.80\%). These results demonstrate that DMDSC consistently provides strong open-set recognition performance across diverse medical imaging datasets while maintaining competitive closed-set classification accuracy. 
As OMCL code is not publicly available,
we report OMCL only on the datasets for which results are published.
\begin{table*}[htbp]
\centering
\caption{Comparison with state-of-the-art methods across five datasets based on ACC, AUROC, and OSCR metrics.
$t$ = number of open-set trials (known and unknown classes are chosen randomly in each trial). Mean results over five runs are reported. \underline{\textbf{Bold-underlined}} values show best results, while only \textbf{bold} values show second best results.}
\label{tab:comparison with SOTA}

\setlength{\tabcolsep}{2pt}
\renewcommand{\arraystretch}{1.5}
\footnotesize
\begin{tabular}{|l|c|c|c|c|c|c|c|c|c|c|c|c|c|c|c|}
\hline
\multirow{2}{*}{Methods}
& \multicolumn{3}{c|}{\text{BloodMNIST, $t=5$}}
& \multicolumn{3}{c|}{\text{OCTMNIST, $t=3$}}
& \multicolumn{3}{c|}{\text{DermaMNIST, $t=4$}}
& \multicolumn{3}{c|}{\text{BreaKHis$(40\times)$, $t=5$}}
& \multicolumn{3}{c|}{ASC, $t=4$} \\
\cline{2-16}
& ACC & AUROC & OSCR
& ACC & AUROC & OSCR
& ACC & AUROC & OSCR
& ACC & AUROC & OSCR
& ACC & AUROC & OSCR \\
\hline
GCPL \cite{yang2018robust}     & 98.10 & 85.50 & 85.00 & 94.80 & 65.50 & 64.20 & 81.78 & 70.37 & 62.53 & 65.66 & 57.47 & 41.10 & 52.24 & 54.88 & 34.60 \\
\hline
RPL \cite{chen2020learning}       & 98.00 & 86.80 & 86.30 & 93.70 & 65.90 & 64.20 & 80.78 & 69.93 & 61.76 & \underline{\textbf{68.17}} & 58.32 & 42.27 & 57.86 & 55.61 & 37.55 \\
\hline
ARPL+CS \cite{chen2021adversarial}  & \textbf{98.50} & 87.60 & 87.10 & 95.90 & 77.70 & 75.80 & \textbf{86.60} & 73.28 & 67.15 & 62.72 & 62.59 & 42.36 & 60.39 & 58.16 & 39.27 \\
\hline
ConOSR \cite{xu2023contrastive}     & \underline{\textbf{99.11}} & 85.14 & 84.88 & \underline{\textbf{97.37}} & \textbf{81.87} & \underline{\textbf{80.82}} & 85.97 & 79.50 & \textbf{73.11} & 63.91 & 62.15 & 43.40 & \textbf{89.70} & 73.66 & 69.17 \\
\hline
OMCL \cite{liu2023learning}     & 98.30 & 88.60 & 88.00 & \textbf{96.80} & 78.90 & 77.80 & -- & -- & -- & -- & -- & -- & -- & -- & -- \\
\hline
DSC \cite{cevikalp2024reaching}       & 97.39 & 89.06 & 88.00 & 91.93 & 78.25 & 74.06 & 85.63 & 79.07 & 71.06 & 63.26 & 64.84 & 44.11 & 86.31 & 70.87 & 63.55 \\
\hline
UCDSC \cite{Aditya_2026_WACV}         & 97.69 & \textbf{89.93} & \textbf{88.84} & 96.70 & 78.11 & 76.56 & 83.23 & \textbf{81.33} & 71.28 & 65.23 & \textbf{65.12} & \textbf{45.50} & 89.12 & \underline{\textbf{74.80}} & \textbf{69.51} \\
\hline
DMDSC (Ours)        & 98.40 & \underline{\textbf{91.12}} & \underline{\textbf{90.49}} & 94.65 & \underline{\textbf{83.60}} & \textbf{80.21} & \underline{\textbf{86.82}} & \underline{\textbf{81.44}} & \underline{\textbf{73.64}} & \textbf{67.48} & \underline{\textbf{66.78}} & \underline{\textbf{46.54}} & \underline{\textbf{93.67}} & \textbf{74.11} & \underline{\textbf{71.38}}   \\
\hline
\end{tabular}

\end{table*}
Standard deviation (std) (over five runs) of Accuracy (ACC), AUROC and OSCR of Table \ref{tab:comparison with SOTA} are reported in Table \ref{tab:std_table} for all five datasets. The ROC curves of DMDSC across different open-set trials for all five datasets are shown in Figure~\ref{fig:roc_curves}, along with the corresponding mean ROC curves from a single run. 
\begin{table*}[htbp]
\centering
\caption{Comparison with state-of-the-art methods across five datasets based on ACC, AUROC, and OSCR metrics. 
$t$ denotes the number of open-set trials. Results are reported as \textbf{standard deviation} ($\sigma$) over five runs.}
\label{tab:std_table}

\setlength{\tabcolsep}{0.5pt}
\renewcommand{\arraystretch}{1.5}
\footnotesize
\begin{tabular}{|l|c|c|c|c|c|c|c|c|c|c|c|c|c|c|c|}
\hline
\multirow{2}{*}{Methods}
& \multicolumn{3}{c|}{\text{BloodMNIST, $t=5$}}
& \multicolumn{3}{c|}{\text{OCTMNIST, $t=3$}}
& \multicolumn{3}{c|}{\text{DermaMNIST, $t=4$}}
& \multicolumn{3}{c|}{\text{BreaKHis$(40\times)$, $t=5$}}
& \multicolumn{3}{c|}{ASC, $t=4$} \\
\cline{2-16}
& $\sigma_{\mathrm{ACC}}$ & $\sigma_{\mathrm{AUROC}}$ & $\sigma_{\mathrm{OSCR}}$
& $\sigma_{\mathrm{ACC}}$ & $\sigma_{\mathrm{AUROC}}$ & $\sigma_{\mathrm{OSCR}}$
& $\sigma_{\mathrm{ACC}}$ & $\sigma_{\mathrm{AUROC}}$ & $\sigma_{\mathrm{OSCR}}$
& $\sigma_{\mathrm{ACC}}$ & $\sigma_{\mathrm{AUROC}}$ & $\sigma_{\mathrm{OSCR}}$
& $\sigma_{\mathrm{ACC}}$ & $\sigma_{\mathrm{AUROC}}$ & $\sigma_{\mathrm{OSCR}}$ \\
\hline
DSC \cite{cevikalp2024reaching}       & 0.38 & 0.35 & 0.46 & 0.92 & 0.52 & 0.21 & 0.81 & 0.53 & 0.72 & 1.40 & 1.08 & 1.33 & 3.26 & 0.97 & 1.17 \\
\hline
UCDSC \cite{Aditya_2026_WACV}         & 0.26 & 0.87 & 0.61 & 0.94 & 0.18 & 0.49 & 1.42 & 0.33 & 1.20 & 1.29 & 0.68 & 0.49 & 0.64 & 1.56 & 1.21 \\
\hline
DMDSC (Ours)        & 0.16 & 0.43 & 0.41 & 0.40 & 0.75 & 0.67 & 0.22 & 0.49 & 0.36 & 0.64 & 0.59 & 0.75 & 0.58 & 0.89 & 0.71   \\
\hline
\end{tabular}

\end{table*}

\subsubsection{Effectiveness of Losses}
We study the effect of varying $\lambda_{\text{inter}}$ and $\lambda_{\text{bg}}$ on ACC, AUROC, and OSCR across BloodMNIST, OCTMNIST, DermaMNIST, BreaKHis$(40\times)$ and ASC 
(Tables~ \ref{tab:bloodmnist_lambda} -- \ref{tab:octmnist_lambda}). On BloodMNIST (Table \ref{tab:bloodmnist_lambda}), AUROC/OSCR improves from 87.23\%/86.75\% to 89.27\%/88.87\% by adding $\mathcal{L}_{inter}$ ($\lambda_{\text{inter}}\neq 0,\, \lambda_{\text{bg}}=0$), to 90.37\%/89.97\% by adding $\mathcal{L}_{\text{bg}}$ ($\lambda_{\text{inter}}=0,\, \lambda_{\text{bg}}\neq 0$), and to 91.12\%/90.49\% when both losses are combined. Similar complementary improvements are observed on DermaMNIST 
(Table \ref{tab:dermamnist_lambda}), BreaKHis$(40\times)$ 
(Table \ref{tab:breakhis_40_lambda}), and ASC
(Table \ref{tab:asc_lambda}). On OCTMNIST 
(Table \ref{tab:octmnist_lambda}) 
adding $\mathcal{L}_{inter}$ ($\lambda_{\text{inter}}\neq 0,\, \lambda_{\text{bg}}=0$) improves AUROC/OSCR from 81.81\%/79.77\% to 81.90\%/79.97\% (+0.09/+0.20\%), and further introducing $\mathcal{L}_{bg}$ gives the best AUROC/OSCR of 83.60\%/80.21\% (+1.79\%/+0.44\%), whereas adding $\mathcal{L}_{\text{bg}}$ ($\lambda_{\text{inter}}=0,\, \lambda_{\text{bg}}\neq 0$) does not improve performance. 
OCTMNIST images exhibit subtle morphological variations that may not be adequately captured by margin-based regularizer, incorporating shape/deformation based priors could further address this limitation.
Table \ref{tab:bg_ablation} isolates auxiliary-data and dynamic-margin effects. DMDSC achieves the best OSCR on all five datasets and AUROC on three datasets without background data.

\begin{table*}[htbp]
\centering
\caption{Ablation study on the effect of background (BG) samples for DSC, UCDSC and DMDSC across five datasets. 
ACC, AUROC, and OSCR are reported in percentage (\%). 
$t$ denotes the number of open-set trials, where known and unknown classes are selected randomly in each trial. 
Mean results over five runs are reported. Bold values indicate the better result between DSC, UCDSC and DMDSC under the same BG setting.}
\label{tab:bg_ablation}

\setlength{\tabcolsep}{2pt}
\renewcommand{\arraystretch}{1.5}
\footnotesize

\begin{tabular}{|l|c|c|c|c|c|c|c|c|c|c|c|c|c|c|c|}
\hline

\multirow{2}{*}{\text{Methods}~\textbf{W/o BG}}
& \multicolumn{3}{c|}{\text{BloodMNIST, $t=5$}}
& \multicolumn{3}{c|}{\text{OCTMNIST, $t=3$}}
& \multicolumn{3}{c|}{\text{DermaMNIST, $t=4$}}
& \multicolumn{3}{c|}{\text{BreaKHis $(40\times)$, $t=5$}}
& \multicolumn{3}{c|}{\text{ASC, $t=4$}} \\
\cline{2-16}

& \text{ACC} & \text{AUROC} & \text{OSCR}
& \text{ACC} & \text{AUROC} & \text{OSCR}
& \text{ACC} & \text{AUROC} & \text{OSCR}
& \text{ACC} & \text{AUROC} & \text{OSCR}
& \text{ACC} & \text{AUROC} & \text{OSCR} \\
\hline

DSC \cite{cevikalp2024reaching}
& \textbf{98.64} & 87.23 & 86.75
& 96.25 & 81.81 & 79.77
& \textbf{87.28} & 70.80 & 64.82
& 66.83 & 62.21 & 43.75
& 84.73 & 69.53 & 62.69 \\
\hline

UCDSC \cite{Aditya_2026_WACV}
& 98.01 & \textbf{89.38} & 88.78
& 93.26 & 79.33 & 75.84
& 83.51 & 77.41 & 68.75
& \textbf{68.82} & 64.45 & 44.87
& 86.30 & \textbf{73.34} & 66.64 \\
\hline

DMDSC (Ours)
& 98.55 & 89.27 & \textbf{88.87}
& \textbf{96.29} & \textbf{81.90} & \textbf{79.97}
& 86.78 & \textbf{81.40} & \textbf{73.46}
& 67.55 & \textbf{65.91} & \textbf{46.12}
& \textbf{92.92} & 72.73 & \textbf{69.56} \\
\hline

\text{Methods \textbf{With BG}}
& \text{ACC} & \text{AUROC} & \text{OSCR}
& \text{ACC} & \text{AUROC} & \text{OSCR}
& \text{ACC} & \text{AUROC} & \text{OSCR}
& \text{ACC} & \text{AUROC} & \text{OSCR}
& \text{ACC} & \text{AUROC} & \text{OSCR} \\
\hline

DSC \cite{cevikalp2024reaching}
& 97.39 & 89.06 & 88.00
& 91.93 & 78.25 & 74.06
& 85.63 & 79.07 & 71.06
& 63.26 & 64.84 & 44.11
& 86.31 & 70.87 & 63.55 \\
\hline

UCDSC \cite{Aditya_2026_WACV}
& 97.69 & 89.93 & 88.84
& \textbf{96.70} & 78.11 & 76.56
& 83.23 & 81.33 & 71.28
& 65.23 & 65.12 & 45.50
& 89.12 & \textbf{74.80} & 69.51 \\
\hline

DMDSC (Ours)
& \textbf{98.40} & \textbf{91.12} & \textbf{90.49}
& 94.65 & \textbf{83.60} & \textbf{80.21}
& \textbf{86.82} & \textbf{81.44} & \textbf{73.64}
& \textbf{67.48} & \textbf{66.78} & \textbf{46.54}
& \textbf{93.67} & 74.11 & \textbf{71.38} \\
\hline

\end{tabular}
\end{table*}

\begin{table*}[htbp]
  \centering
  \setlength{\tabcolsep}{1.5pt}    %
  \renewcommand{\arraystretch}{1.1} %
   \caption{Results on Datasets with varying $\lambda_{inter}$ and $\lambda_{bg}$.}
  \begin{subtable}[t]{0.45\textwidth}
    \centering
    \caption{BloodMNIST}
    \label{tab:bloodmnist_lambda}
    \footnotesize
    \resizebox{\linewidth}{!}{%
      \begin{tabular}{|c|c|c|c|c|c|c|c|c|c|}
        \hline
        $\lambda_{inter}$ & 0 & 0.01 & 0.1 & 1 & 10 & 0 & \textbf{0.1} & 0.1 & 0.1 \\
        \hline
        $\lambda_{bg}$ & 0 & 0 & 0 & 0 & 0 & 0.1 & \textbf{0.1} & 1 & 10 \\
        \hline
        ACC & 98.64 & 98.91 & 98.55 & 97.41 & 97.52 & 98.87 & \textbf{98.40} & 97.88 & 96.99 \\
        AUROC & 87.23 & 88.79 & 89.27 & 89.55 & 88.61 & 90.37 & \textbf{91.12} & 90.53 & 89.18 \\
        OSCR & 86.75 & 88.45 & 88.87 & 88.42 & 87.43 & 89.97 & \textbf{90.49} & 89.63 & 87.57 \\
        \hline
      \end{tabular}%
    }
  \end{subtable}\hfill
  \begin{subtable}[t]{0.45\textwidth}
    \centering
    \caption{ASC}
    \label{tab:asc_lambda}
    \footnotesize
    \resizebox{\linewidth}{!}{%
      \begin{tabular}{|c|c|c|c|c|c|c|c|c|c|}
        \hline
        $\lambda_{inter}$ & 0 & 0.1 & 1 & 10 & 20 & 0 & 10 & \textbf{10} & 10 \\
        \hline
        $\lambda_{bg}$ & 0 & 0 & 0 & 0 & 0 & 1 & 0.1 & \textbf{1} & 10 \\
        \hline
        ACC & 84.73 & 89.05 & 93.11 & 92.92 & 92.23 & 90.60 & 93.00 & \textbf{93.67} & 93.11 \\
        AUROC & 69.53 & 70.56 & 71.87 & 72.73 & 72.43 & 70.52 & 72.95 & \textbf{74.11} & 73.98 \\
        OSCR & 62.69 & 66.14 & 68.82 & 69.56 & 68.93 & 66.31 & 69.89 & \textbf{71.38} & 70.76 \\
        \hline
      \end{tabular}%
    }
  \end{subtable}

  \vspace{6pt}

  \begin{subtable}[t]{0.45\textwidth}
    \centering
    \caption{DermaMNIST}
    \label{tab:dermamnist_lambda}
    \footnotesize
    \resizebox{\linewidth}{!}{%
      \begin{tabular}{|c|c|c|c|c|c|c|c|c|c|}
        \hline
        $\lambda_{inter}$ & 0 & 0.01 & 0.1 & 1 & 10 & 0 & \textbf{0.1} & 0.1 & 0.1   \\
        \hline
        $\lambda_{bg}$ & 0 & 0 & 0 & 0 & 0 & 0.001 & \textbf{0.001} & 0.01 & 0.1   \\
        \hline
        ACC & 87.28 & 85.69 & 86.78 & 86.04 & 84.69 & 87.04 & \textbf{86.82} & 86.67 & 86.40  \\
        AUROC & 70.80 & 79.75 & 81.40 & 80.33 & 79.57 & 79.14 & \textbf{81.44} & 81.43 & 81.54  \\
        OSCR & 64.82 & 71.76 & 73.46 & 72.0 & 70.27 & 72.26 & \textbf{73.64} & 73.48 & 73.52  \\
        \hline
      \end{tabular}%
    }
  \end{subtable}\hfill
  \begin{subtable}[t]{0.42\textwidth}
    \centering
    \caption{BreaKHis$(40\times)$}
    \label{tab:breakhis_40_lambda}
    \footnotesize
    \resizebox{\linewidth}{!}{%
      \begin{tabular}{|c|c|c|c|c|c|c|c|c|}
        \hline
        $\lambda_{inter}$ & 0 & 0.1 & 1 & 10 & 0 & 1 & \textbf{1} & 1 \\
        \hline
        $\lambda_{bg}$ & 0 & 0 & 0 & 0 & 0.1 & 0.01 & \textbf{0.1} & 1 \\
        \hline
        ACC & 66.83 & 66.52 & 67.55 & 66.99 & 66.95 & 67.59 & \textbf{67.48} & 67.64 \\
        AUROC & 62.21 & 65.41 & 65.91 & 65.29 & 64.82 & 65.21 & \textbf{66.78} & 63.17 \\
        OSCR & 43.75 & 45.53 & 46.12 & 45.64 & 45.27 & 45.83 & \textbf{46.54} & 43.79 \\
        \hline
      \end{tabular}%
    }
  \end{subtable}
  \vspace{6pt}

  \begin{subtable}[t]{0.72\textwidth}
    \centering
    \caption{OCTMNIST}
    \label{tab:octmnist_lambda}
    \footnotesize
    \resizebox{\linewidth}{!}{%
      \begin{tabular}{|c|c|c|c|c|c|c|c|c|c|c|c|c|c|c|}
        \hline
        $\lambda_{inter}$ & 0 & 1 & 0.1 & 0.01 & 0.001 & 0.0001 & $10^{-5}$ & $10^{-6}$ & 0 & $10^{-5}$ & $10^{-5}$ & $10^{-5}$ & \textbf{$10^{-5}$} & $10^{-5}$\\
        \hline
        $\lambda_{bg}$ & 0 & 0 & 0 & 0 & 0 & 0 & 0 & 0 & 1 & 0.01 & 0.001 & 0.0001 & \textbf{1} & 10 \\
        \hline
        ACC & 96.25 & 93.44 & 93.70 & 94.70 & 96.70 & 96.61 & 96.29 & 96.28 & 92.70 & 96.51 & 96.40 & 95.79 & \textbf{94.65} & 94.36 \\
        AUROC & 81.81 & 80.07 & 79.91 & 82.08 & 80.89 & 80.47 & 81.90 & 81.39 & 81.61 & 81.59 & 81.09 & 81.55 & \textbf{83.60} & 81.61 \\
        OSCR & 79.77 & 76.36 & 76.19 & 78.91 & 79.19 & 78.83 & 79.97 & 79.44 & 76.99 & 79.78 & 79.34 & 79.55 & \textbf{80.21} & 78.22 \\
        \hline
      \end{tabular}%
    }
  \end{subtable}
\end{table*}
\subsubsection{Performance under Class Imbalance}
As shown in Table \ref{tab:BAcc_macro_f1}, DMDSC consistently improves balanced accuracy (BAcc) over UCDSC across all datasets and achieves higher Macro-F1 scores. The improvements are particularly substantial on the highly imbalanced DermaMNIST (IR=58.58) and BreaKHis(40x) (IR=7.93) datasets, where DMDSC improves BAcc/Macro-F1 by 10.03\%/10.43\% and 11.82\%/11.27\%, respectively. Improvements are also observed on BloodMNIST (IR=2.74) and OCTMNIST (IR=5.78). For the nearly balanced ASC dataset, both methods achieve almost identical performance. These results indicate that the dynamic margin is particularly beneficial under class imbalance, promoting more balanced recognition across majority and minority classes.

\begin{table*}[htbp]
\centering
\caption{Comparison of the proposed DMDSC with UCDSC across five datasets based on BAcc and Macro-F1 metrics.
$t$ denotes the number of open-set trials.}
\label{tab:BAcc_macro_f1}

\setlength{\tabcolsep}{2pt}
\renewcommand{\arraystretch}{1.5}
\footnotesize

\begin{tabular}{|l|c|c|c|c|c|c|c|c|c|c|c|c|}
\hline

\multirow{2}{*}{Methods}
& \multicolumn{2}{c|}{\text{BloodMNIST, $t=5$}}
& \multicolumn{2}{c|}{\text{OCTMNIST, $t=3$}}
& \multicolumn{2}{c|}{\text{DermaMNIST, $t=4$}} 
& \multicolumn{2}{c|}{\text{BreaKHis$(40\times)$, $t=5$}} 
& \multicolumn{2}{c|}{ASC, $t=4$} \\
\cline{2-11}

& BAcc & \mbox{Macro-F1}
& BAcc & \mbox{Macro-F1}
& BAcc & \mbox{Macro-F1}
& BAcc & \mbox{Macro-F1}
& BAcc & \mbox{Macro-F1}
\\
\hline

UCDSC \cite{Aditya_2026_WACV}
& 96.32 & 96.51
& 82.84 & 83.49
& 52.11 & 53.93
& 35.65 & 32.23
& 92.88 & 92.79
\\
\hline

DMDSC (Ours)
& \textbf{98.13} & \textbf{98.23}
& \textbf{84.14} & \textbf{86.10}
& \textbf{62.14} & \textbf{64.36}
& \textbf{47.47} & \textbf{43.50}
& \textbf{92.90} & 92.78
\\
\hline

\end{tabular}
\end{table*}

\subsubsection{Robustness Evaluation on ASC Dataset}
Table \ref{tab:Imbalance_study} compares the performance of DSC, UCDSC, and the proposed DMDSC under class imbalance ratios (IR) of 10 and 100. Under IR~=~10, DMDSC achieves the best performance across all metrics, improving ACC to 77.25\%, AUROC to 68.33\%, and OSCR to 57.59\%, outperforming DSC by 4.26\%, 2.12\%, and 4.00\%, and UCDSC by 3.53\%, 1.39\%, and 2.91\% respectively. Under the more severe class imbalance setting (IR~=~100), DMDSC consistently outperforms both baselines, achieving 57.89\% ACC, 60.93\% AUROC, and 41.39\% OSCR in the similar way. These results demonstrate that the proposed dynamic-margin formulation effectively improves robustness to class imbalance while providing superior open-set recognition performance.
\begin{table}[t]
\centering
\caption{Comparison of DSC, UCDSC, and the proposed DMDSC on the ASC dataset with varying imbalance ratios (IR). Performance is evaluated using ACC, AUROC, and OSCR metrics.}
\label{tab:Imbalance_study}
\footnotesize
\setlength{\tabcolsep}{2pt}
\renewcommand{\arraystretch}{1.5}

\begin{tabular}{|l|c|c|c|c|c|c|}
\hline
\multirow{2}{*}{\textbf{Methods}}
& \multicolumn{3}{c|}{\textbf{IR = 10}}
& \multicolumn{3}{c|}{\textbf{IR = 100}} \\
\cline{2-7}
& \text{ACC} & \text{AUROC} & \text{OSCR}
& \text{ACC} & \text{AUROC} & \text{OSCR} \\
\hline
DSC \cite{cevikalp2024reaching}
& 72.99 & 66.21 & 53.59
& 52.75 & 59.25 & 36.51 \\
\hline
UCDSC \cite{Aditya_2026_WACV}
& 73.72 & 66.94 & 54.68
& 52.90 & 60.82 & 37.49 \\
\hline
DMDSC (Ours)
& \textbf{77.25} & \textbf{68.33} & \textbf{57.59}
& \textbf{57.89} & \textbf{60.93} & \textbf{41.39} \\
\hline
\end{tabular}%
\end{table}
\begin{table}[t]
\centering
\caption{Hyperparameter values used for each dataset. BL: BloodMNIST, OCT: OCTMNIST, DE: DermaMNIST, ASC: Augmented Skin Conditions, BK: BreaKHis ($40\times$).}
\label{tab:optimal_hyperparameters}

\footnotesize
\setlength{\tabcolsep}{2pt}
\renewcommand{\arraystretch}{1.5}

\begin{tabular}{|l|c|c|c|c|c|}
\hline
\textbf{Parameter}
& \textbf{BL}
& \textbf{OCT}
& \textbf{DE}
& \textbf{BK}
& \textbf{ASC}\\
\hline
Batch size & 64 & 128 & 64 & 8 & 16\\
\hline
$m_{\min}$ & 35 & 35 & 35 & 35 & 35 \\
\hline
$m_{\max}$ & 55 & 55 & 55 & 55 & 55 \\
\hline
$\lambda_{\mathrm{inter}}$ & 0.1 & $10^{-5}$ & 0.1 & 1 & 10\\
\hline
$\lambda_{\mathrm{bg}}$ & 0.1 & 1 & 0.001 & 0.1 & 1 \\
\hline
$R$ & 100 & 100 & 100 & 100 & 100 \\
\hline
\end{tabular}
\end{table}

\section{Conclusion}
\label{sec:conclusion}

In this work, we proposed DMDSC, a Dynamic-Margin Deep Simplex Classifier
for open-set recognition in imbalanced medical image datasets. Unlike existing
Neural-Collapse based approaches that rely on a uniform margin, our method introduces a class-aware dynamic margin that scales inversely with class frequency,
thereby strengthening the representations of rare pathologies. Within the simplex ETF feature space, the intra-class loss encourages compact clustering around the corresponding class centers, while the class-adaptive dynamic margin improves inter-class and background separation.
Experimental results on BloodMNIST, OCTMNIST, DermaMNIST, BreaKHis$(40\times)$ and ASC show that the proposed approach improves AUROC and OSCR
while maintaining strong accuracy, particularly under severe class imbalance.
The findings indicate that incorporating a dynamic margin improves feature separation under class imbalance without increasing model complexity. This work
can be useful in the healthcare domain where AI systems encounter rare diseases,
imaging artifacts, or previously unseen pathologies, enabling safer diagnosis by
accurately classifying known diseases while reliably rejecting unknown cases. As
a part of future work, we also plan to extend our method to non-linear margin
function m(p) to understand its impact on the OSR performance.
We also plan to include shape/deformation based priors to capture subtle morphological patterns in medical image datasets.
\section*{Data and Code Availability}

\begin{sloppypar}
The datasets used in this study are publicly available. BloodMNIST, OCTMNIST and
DermaMNIST are part of the MedMNIST v2 collection (\url{https://medmnist.com/}).
The BreakHis dataset, including the $40\times$ subset used in this study, is available
at \url{https://web.inf.ufpr.br/vri/databases/breast-cancer-histopathological-database-breakhis/}.
The Augmented Skin Conditions (ASC) dataset is available here
(\url{https://drive.google.com/file/d/1sa8w_1flYl9bJFGmUVkUGdFrEkeiT9oP/view?usp=sharing}).
The auxiliary 300K Random Images dataset is accessible through the Outlier Exposure
repository (\url{https://github.com/hendrycks/outlier-exposure}). Dataset preparation
and experimental protocols are described in Section~\ref{sec:exp_res_disc}. No new
primary data were collected for this study.
The source code can be accessed at \href{https://github.com/vishal1919/DMDSC}{github.com/vishal1919/DMDSC}.
\end{sloppypar}

\section*{Funding}

This research received no external funding.

\bibliography{sample}
\end{document}